\documentclass{amsart}
\newcommand{\R}{\ensuremath{\mathbb{R}}}
\usepackage{graphicx}
\usepackage[capitalise,noabbrev]{cleveref}
\usepackage{tabularx}
\usepackage{tikz}
\usepackage{subcaption}
\usepackage{bbm}
\usepackage[italicdiff]{physics}

\newtheorem{lemma}{Lemma}
\DeclareMathOperator{\Var}{Var}

\title{A new model for natural groupings in high-dimensional data}

\author{Mireille Boutin} \address{Department of Mathematics and
  Computer Science, Eindhoven Institute of Technology, 5600 MB
  Eindhoven, Netherlands, and Department of Mathematics s and Elmore Family School of Electrical and Computer Engineering, Purdue
  University, 150 N.~University St., West Lafayette, IN, USA 47907}
\email{m.boutin@tue.nl}%
\author{Evzenie Coupkova} \address{Department of Mathematics, Purdue
  University, 150 N.~University St., West Lafayette, IN, USA 47907} \email{ecoupkov@purdue.edu}


\begin{document}
	\maketitle

\begin{abstract}
Clustering aims to divide a set of points into groups. The current paradigm assumes that the grouping is well-defined (unique) given the probability model from which the data is drawn. 
Yet, recent experiments have uncovered several high-dimensional datasets that form different binary groupings  after projecting the data to randomly chosen one-dimensional subspaces. This paper describes a probability model for the data that could explain this phenomenon. 
It is a simple model to serve as a proof of concept for understanding the geometry of high-dimensional data.
We start by building a rescaled multivariate Bernouilli model (stretched hypercube) so to create several overlapping grouping structures in the data. The size of each scaling parameter is related to the likelihood of uncovering the corresponding grouping by random 1D projection.
 Clusters in the original space are then created by adding noise to this cluster-free model.
 In high dimension, these clusters would hardly be observable given a sample set from the distribution because of the curse of dimensionality, but the binary groupings are clear. Our construction makes it clear that one needs to make a distinction between ``groupings" and ``clusters" in the original space.
It also highlights the need to interpret any clustering found in projected data as merely one among potentially many other groupings in a dataset.

\end{abstract}

\section{Introduction} \label{sec:intro}
Unsupervised machine learning seeks to use random samples to build a model that encapsulates 
the structure of the probability model underlying a dataset. A common assumption is that the data is distributed following a mixture of uni-modal probability densities such as Gaussians. In low-dimension, the sample points of such a model tend to bundle together in space when they are drawn from the same uni-modal probability density. Thus, in low-dimension, a natural grouping of the points is often found by looking for regions of space in which points accumulate, in other words ``clusters".

But high-dimensional datasets are typically very sparse and thus free of clusters.  
So one typically groups the points by looking for clusters within low-dimensional projections of the dataset. It is well understood that projecting a dataset may destroy its structure. For example, a mixture of two Gaussians can become a unimodal distribution after projection. 
However, projecting can also create structures. For example, a probability density function consisting of a large number of Dirac deltas  sparsely distributed along two straight parallel lines would become a mixture of two deltas after projection on the plane perpendicular to the lines. Thus, sample
points that do not form an agglomeration in the high-dimensional dataset may become agglomerated after projection to a lower-dimension. Such agglomerations are not ``ghost" ones resulting from the sampling process; they result from structures that are present in the projected density function, and thus correspond to valid groupings among the sample points.
One way to explain this is by viewing  
the projection direction as a  dimension containing information, and the remaining dimensions as containing noise. Since there are many ways to project a dataset, that is to say many ways to choose what is noise versus what is information, there can be many ways to create such groupings, and the different groupings so obtained need not be consistent. This means that points drawn from a high-dimensional density function may be grouped in different ways, and the different ways to group them, though inconsistent, may all be correct. In that scenario, a grouping is different from a clustering in the original space, in that a cluster is an accumulation of points in space, whereas a grouping is merely a meaningful division of the points into subsets. While the points may cluster (i.e., accumulate) after projection, thus defining a grouping, the points in the original space are not necessarily clustered in the sense of being clumped together.

This phenomenon was  initially observed experimentally in \cite{han2015hidden} for image datasets. More specifically, the feature points used to represent images were found to often form bimodal histograms after a projection on a random line. Later work appeared to confirm the validity of these bimodal structures. In particular, a hierarchical clustering method consisting of a tree of random projections followed by thresholding turned out to be surprisingly accurate at clustering some high-dimensional benchmark datasets in \cite{YellamrajuTarun2018CaCo}. We are currently investigating the advantage of clustering by random projection for small datasets \cite{bradford2020clustering}. We are also studying the use of different grouping structures to analyze dependencies between different variables in a small dataset \cite{yellamraju2018pattern} and applying these ideas to analyzing educational data \cite{williams2018characterizing,YellamrajuMaganaBoutin2019,magana2019role,CruzCastroMaganaDouglasBoutin2021}.

For this paper, we asked ourselves: 
what kind of structure in the probability density model of the high-dimensional data would cause 
this to occur? In other words, what kind of density model would a high-dimensional dataset need to be drawn from in order for the points to have a high probability of forming clusters after a projection on a random line.

Again, this questioning is not about
 ``ghost" structures, that is to say observations due to the fact that we are observing only a sample set, as opposed to the entire projected distribution. 
Indeed in \cite{bickel2018projection}, Bickel, Kur, and Nadler highlight the fact that, when projected to a lower-dimensional space, data points in high-dimension may appear to form structures that do not actually exist. In particular, they show that with a sample of points drawn independently and identically from a high-dimensional Gaussian distribution, such ghost structures can be created to mimic any distribution by choosing a projection specifically tailored to that sample.
However, results by Diaconis and Freedman \cite{DiaconisPersi1984AoGP} imply that such projections must be carefully chosen, as most projections of a data set (satisfying some mild conditions) appear to be Gaussian. The fact that the binary clustering structures appear after a {\em random} choice of projection direction would therefore contradict the ghost structure hypothesis.

Let us summarize. If the experimental observations of \cite{han2015hidden,YellamrajuTarun2018CaCo} are correct, then there must exist a probability density function from which one can draw random samples and expect (i.e., not on accident) the points to be such that:
\begin{enumerate}
	\item a projection of the sample points on a random one-dimensional subspace is likely to yield a binary clustering of the points;
	\item the collection of groupings obtained by clustering with respect to different random 1D projections are not consistent with the existence of a unique clustering of the points in the high-dimensional space, namely
	\subitem - their projection directions are different;
	\subitem - the groupings of the points they define are different; 
	\subitem -  the within-group distances, in the original space, are not significantly smaller than the between-group distances. 
\end{enumerate} 
Naturally, this probability density function would have to be such that datasets drawn from it would be expected to violate the assumptions of Diaconis and Freedman \cite{DiaconisPersi1984AoGP}.

This paper describes such a probability density function. We construct it in two stages. 
The first stage, described in Section \ref{sec:stretch_box}, yields an idealized probability density function that serves as a noiseless skeleton for the final probability density function. 
This idealized density is a sparse mixture of Dirac deltas in ${\mathbb R}^D$ built from a multivariate Bernoulli model. The locations of the deltas are chosen in such a way that they do not form any clusters. Yet, a projection of the density function on a random one-dimensional subspace is likely to yield a clear binary clustering. This is accomplished by placing the deltas on the vertices of a hypercube (Section \ref{sec:hypercube}) whose edges are then stretched (Section \ref{sec:stretch_box}) following a geometrical progression. 
The final probability density is obtained (Section \ref{sec:general_model}) by adding noise to the vertices and
translating, rotating and slightly skewing the density with an affine transform.

The resulting probability density  model should be viewed as a proof of concept that illustrates the kind of structures found in real datasets. More sophisticated probability models can be build from it (e.g., hierarchical mixture models). Still, it is possible to construct interesting datasets from it. In Section \ref{sec:experiments}, We illustrate this by showing how to build a dataset of image samples starting from a small image dictionary and a background image. First we show how to use the image dictionary to build a (affine transformed) stretched hypercube probability model. This first model can be used to form random images without noise. We show how adding noise to the model modifies the synthetic images while preserving the grouping structure.
Experiments with a real image dataset show how different clusterings observed after 1D projection of the data provide an overlapping grouping structure for the dataset. We also observe that such grouping structures may be organized in a hierarchical fashion.

\section{Model Construction Overview}
We are looking for a probability density function on ${\mathbb R}^D$ from which we can draw random samples with clustering properties (1) and (2)  mentioned in the introduction. To be precise, here we define a clustering as a partition of a set of points such that the within-class scatter (i.e., the sum of the variances of each class weighted by their relative number of points) is less than the between-class scatter (i.e., the sum of the squared distances between the mean of each class and the dataset mean, weighted by their relative number of points.) A binary clustering of a set of points is considered to be valid (i.e., correct) if 
it is the result of an existing structure in the probability model, rather than a random draw of samples. In other words, a projection of the density itself exhibits the same bimodality.

We begin in Section \ref{sec:two_gaussian}  by constructing a model which will generate samples sets that will typically satisfies Property (1), but not Property (2). The model is a mixture of Gaussian distributions whose means are far away from each other.

Our next step is in 
Section \ref{sec:hypercube}, where we describe a model whose sample sets will typically satisfy Property (2) but not Property (1). 
The model is a a multivariate Bernoulli distribution, whose density is a
mixture of Dirac deltas 
placed on the vertices of a hypercube in ${\mathbb R}^D$. The carefully chosen sparsity of the delta locations creates a model with no natural clustering in the original space. Yet, several directions of  projection (spanned by the edges of the hypercube) would yield a mixture of Dirac deltas in 1D with a binary clustering structure.

The third step in our construction is in Section \ref{sec:stretch_box},  where we use the knowledge gained from our first model in Section \ref{sec:two_gaussian} to rescale the edges of the hypercube model of the second step so to satisfy (1). The key to the construction is a geometric progression in the lengths of the different edges of the hypercube. This yields a skeleton to which noise is added, and which is translated, rotated, and skewed slightly to yield our final proof of concept model in Section \ref{sec:general_model}.

\section{A highly likely clusterable Model}
\label{sec:two_gaussian}
It is not too difficult to construct a set of points whose projection onto a random line through the origin is likely to be bimodal when the line direction is chosen following a uniform distribution on the sphere. A simple way to do this is to draw the points from a mixture of two Gaussians whose means are far away enough with respect to the space dimension.

Consider a scale mixture of two identical, spherical Gaussians with equal priors, separated by some distance $a$. A random variable $X$ drawn from this mixture 
can be written as $X=\mathcal{N}_D+a\mathbf{e} Y$, where $\mathcal{N}_D$ is a standard normal random variable in ${\mathbb R}^D$, $Y$ is a Bernoulli random variable with $p=\frac{1}{2}$, and $\mathbf{e}\in{\mathbb R^D}$ is a unit vector. The random variable $Y$ can be thought of as the signal, and the random variable $\mathcal{N}_D$ can be thought of as symmetric, Gaussian-distributed noise. As we show below, if $a$ is large enough, we can use random projection to uncover the presence of the signal $Y$ within the noise $\mathcal{N}_D$.

In order to uncover the structure of $X$ by projecting onto a vector $\mathbf{v}$,
some component of $\mathbf{v}$ must be in the direction of $\mathbf{e}$. Suppose $\mathbf{v}$ is sampled uniformly from the zero-centered unit sphere sitting in $\R^D$. As the following lemma shows, in high dimensions $D$, the distribution of $\mathbf{v}\cdot\mathbf{e}$ becomes very sharply peaked at 0.

\begin{lemma}
\label{lemma:distribution_dot_product}
Suppose $\{D_k\}$ is a sequence with $D_k\to\infty$ and suppose $\{a_k\}_{k=1}^{\infty}$ is a sequence such that $a_k/\sqrt{D_k}\to a_0$. As $k$ approaches infinity, the random variable $a_k\mathbf{v}\cdot\mathbf{e}$ converges in distribution to a normally distributed random variable,
\[
\lim_{k\rightarrow \infty }a_k\mathbf{v}\cdot\mathbf{e} = a_0 \mathcal{N}_1.
\]
\end{lemma}
\begin{proof}
Without loss of generality, assume $\mathbf{e}=(1,0,\dots,0)$. One way of selecting a random direction $\mathbf{v}$ is to take $\mathbf{v}=(N_1,\dots,N_{D_k})/|(N_1,\dots,N_{D_k})|$ with $\{N_i\}$ independent normal random variables. Then:
\[a_k\mathbf{v}\cdot\mathbf{e}= \frac{a_k}{\sqrt{D_k}} N_1\sqrt{\frac{D_k}{N_1^2+\dots+N_{D_k}^2}}.\]
By the law of large numbers $\frac{D_k}{N_1^2+\dots+N_{D_k}^2}$ converges in distribution to $1$. 

Now recall Slutsky's theorem, which states in particular that for two sequences of random variables $\left\{X_n\right\}$ and $\left\{Y_n\right\}$, if $X_n\to X$ and $Y_n\to c$ in distribution for some distribution $X$ and some constant $c$, then $X_nY_n\to Xc$  in distribution. Apply this theorem with $X_k= N_1$ and $Y_k=\frac{a_k}{\sqrt{D_k}}\sqrt{\frac{D_k}{N_1^2+\dots+N_{D_k}^2}}$, and the proof is complete.

\end{proof}

We can interpret this intuitively as saying that for a fixed dimension $D$, the random variable $\mathbf{v}\cdot\mathbf{e}$ is approximately distributed as $\frac{1}{\sqrt{D}}\mathcal{N}_1$, that is, a normal random variable with mean 0 and variance $1/D$. This quantifies the sense in which random vectors in high dimensions are ``mostly nearly orthogonal.''

In order to quantify whether random projection is useful, we take a Bayesian perspective. After projection, one can use a thresholding to classify a sample point $x$ as corresponding to either $Y=1$ or $Y=0$. The probability of error of that classification depends on the threshold value. We use the minimum probability of error $E$ over all possible threshold values in order to quantify to what extent this projection is divided into two clusters.  

Since $\mathcal{N}_d$ is spherically symmetric, the distribution of $\mathbf{v}\cdot\mathcal{N}_d$ is simply $\mathcal{N}_1$, a standard normal distribution. Then $\mathbf{v}\cdot X$ is distributed as $\mathcal{N}_1+Ya(\mathbf{e}\cdot \mathbf{v})$. This gives us the following simple expression for $E$.

\begin{lemma}
For any projection vector $\mathbf{v}$, the minimum probability of error $E$ is given by $\Phi(-\frac{1}{2}a|\mathbf{e}\cdot\mathbf{v}|)$, where $\Phi$ is the cumulative distribution function of a unit normal random variable.
\end{lemma}
\begin{proof}[Outline of proof]
Regardless of $\mathbf{v}$, one of the clusters after projection will be a unit normal distribution centered at the origin. The other cluster will also be a unit normal distribution, but its center depends on $a$ and $\mathbf{v}$ as $a\mathbf{e}\cdot\mathbf{v}$. The ideal threshold will be at their midpoint $\frac{1}{2}a\mathbf{e}\cdot\mathbf{v}$, and this produces the stated error.
\end{proof}

By \cref{lemma:distribution_dot_product}, 
$\mathbf{e}\cdot\mathbf{v}$ is distributed approximately as $\frac{1}{\sqrt{D}}\mathcal{N}_1$, so $E$ is approximately distributed as $\Phi(-\frac{a}{2\sqrt{D}}|\mathcal{N}_1|)$. This distribution is skewed toward values close to its maximum of $\frac{1}{2}$ when $\frac{a}{2\sqrt{D}}<1$, and skewed toward values close to its minimum of $0$ when $\frac{a}{2\sqrt{D}}>1$.

Notice $a^2/4$ is precisely the between-class scatter defined in the introduction. The within-class scatter is precisely $D$. Hence, the distribution of $E$ is skewed toward zero precisely when the the between-class scatter is more than the within-class scatter. That is, we find the projection reliably by random projection when the two Gaussians form clusters in $\mathbb{R}^D$ in the traditional sense.

\section{A model with no clusters but clusterable in many different ways}
\label{sec:hypercube}

In the model considered in \cref{sec:two_gaussian}, the noise was many-dimensional and the underlying structure to be discovered was one-dimensional. Here we present a model in $\R^D$ that has meaningful structure in each of its dimensions. As a result, the structure can be divided into two groups in at least $D$ different ways by projection. However, the structure itself contains no cluster in $\R^D$.

Let $Y=(Y_1,Y_2,\dots,Y_D)$ be composed of Bernoulli random variables $Y_i$ (i.e., a multivariate Bernoulli distribution). For now, we assume that they all have the same parameter $p=\frac{1}{2}$ and that they are independent, but these assumptions will be relaxed later in the full model. The multivariate random variable $Y$ takes as its values the vertices of a unit hypercube in $\mathbb{R}^D$. Thus, its probability density function can be viewed as a mixture of Dirac deltas, where the deltas are situated on the vertices of a unit hypercube. 
Since a hypercube in a $D$-dimensional space has $2^D$ vertices, a very large number of non-repeated sample points can potentially be drawn following such a model. For example, in dimension $D=20$, 
over a million points could be put on the different vertices of the hypercube without any overlap.

Because of our choice of $p=\frac{1}{2}$ for all the Bernoulli random variables $Y_i$, the Dirac deltas are equally weighted in the model. In other words, the random variable $Y$ has equal probability of lying on each of the vertices of the hypercube. This encapsulates the idea that there are many (independent) features in the model, which we hope to recover by projecting onto a subspace. Specifically, if we project onto one of the axes, say axis $i$, then we would get a perfectly split bimodal distribution,
 that is to say a perfect classification for the two classes $Y_i=0$ and $Y_i=1$. For any of these projections, the value of the minimum classification error by thresholding is $E=0$.

However, splitting the model after projection onto an axis would separate some points which are only one unit apart, while placing within a group points which are much farther apart, up to a maximum distance of $\sqrt{D}$. Thus the clusters obtained by projection do not correspond to a clustering in $\R^D$. In fact, any separation of the points into groups would have to intersect with at least one edge between two adjacent vertices, and thus by the same argument the dataset itself does not contain any cluster in $\R^D$.

This situation does not present a high likelihood of finding structure by random projection. Indeed, the model fits the assumption of Theorems 1.1 and 1.2 in Diaconis and Freedman \cite{DiaconisPersi1984AoGP}, which states that one-dimensional projections of points sampled from this distribution will generally be indistinguishable from points sampled from a Gaussian distribution.

\section{A model with no cluster but likely clusterable}
\label{sec:stretch_box}

\subsection{Stretched hypercube model construction}
Now we modify the model of \cref{sec:hypercube} in such a way to increase the probability of finding a good linear separation by random projection. As stated earlier, we are motivated by empirical evidence \cite{han2015hidden} that linear separations are quite common; in order to ensure that they are likely in this model, we must violate one or both of the assumptions in Diaconis and Freedman \cite{DiaconisPersi1984AoGP}. 
This is done by carefully stretching the previous model. 
Specifically, let $X=(a_1Y_1,a_2Y_2,\dots,a_DY_D)$.

To simplify the  discussion in this section, we identify our model with a dataset containing one point on each of the vertices of a stretched hypercube in ${\mathbb R}^D$, with side lengths $a_1,a_2,\ldots, a_D$. These points represent the location of the Dirac deltas in the mixture, in other words, the possible locations of sample points drawn from the mixture. Because of our choice of Bernoulli parameter ($p=\frac{1}{2}$), a sample point is equally likely to lie in any of the vertices of the hypercube, thus we put exactly one point per vertex.

As a way of fixing scale, let $a_1=1$. To motivate our choice of the other $a_i$'s, recall that we consider a separation of the points to be a cluster if the within-class scatter is smaller than the between-class scatter. In our specific case, we would consider the split along axis $k$ to be a binary clustering of the modified hypercube if $a_k^2 > \sum_{i\neq k} a_i^2$. One way of ensuring this never occurs is to simply take $a_k^2 = \sum_{i<k} a_i^2$. This gives us a recursive definition for each $a_k$, which resolves to the solution $a_k^2 = 2^{k-2}$ for $2\leq k\leq D$. Thus, a geometric progression of $a_k$ presents itself naturally. We consider below a generalization of this, where $a_k^2 = r^{k-2}$ for some $1\leq r\leq 2$. When $r=1$ the model is precisely the hypercube considered in \cref{sec:hypercube}. When $r>2$, there is a cluster formed by separating the data along dimension $D$, hence we exclude this case.

For example, let us look at the three-dimensional case $D=3$. We then have $a_1=a_2=1$ and $a_3=\sqrt{r}$, and so the dataset contains the eight points:
\begin{align*}
p_1&=(0,0,0) &   p_5&=(1,0,0)\\
p_2&=(0,0,\sqrt{r}) & p_6&=(1,0,\sqrt{r})\\
p_3&=(0,1,0) & p_7&=(1,1,0)\\
p_4&=(0,1,\sqrt{r}) & p_8&=(1,1,\sqrt{r}).
\end{align*}
Projection can yield one of three different clusterings: either according to the first dimension $\{ p_1,p_2,p_3,p_4\}$, $ \{p_5,p_6,p_7,p_8\}$ or according to the second dimension $\{ p_1,p_2,p_5,p_6\}$, $ \{p_3,p_4,p_7,p_8\}$ or according to the third dimension $\{p_1,p_3,p_5,p_7\}$, $ \{p_2,p_4,p_6,p_8\}$. Of these, the last clustering is the most pronounced, since the distance between the clusters is $\sqrt{r}$. However, the within-class scatter is $\frac{1}{4}(a_1^2+a_2^2)=\frac{1}{2}$ while the between-class scatter is $\frac{1}{4}a_3^2=r/4$. Since we imposed $r\leq 2$ we know the within-class scatter is larger, and so the clusters are not separated enough to meet our clustering criterion.

The above argument shows that this model does not have any clusters. However, as we will now demonstrate, the probability of finding a highly separable projection at random is not small. Although many types of clusterings may be possible after projection, we will restrict our attention to clusterings which correspond to one of the $Y_k$, that is, where in one class we have $Y_k=0$ and in the other class we have $Y_k=1$. 

Suppose we project via a vector $\mathbf{v}=(N_1,\dots,N_D)$ where each $N_i$ is a standard Gaussian. We will consider the probability that this produces a good separation corresponding to $Y_k$, for some $k\in \{1,2,\ldots,D\}$. Note that after projection, each data point is of the form $\sum_{i=1}^D a_i Y_i N_i$. Treating $a_i$ and $N_i$ as fixed constants, the within-class scatter with respect to $Y_k$ will be 
\[
\Var\left(\sum_{i\neq k} a_i Y_i N_i\right)=\sum_{i\neq k} (a_i N_i)^2 \Var(Y_i)=\frac{1}{4}\sum_{i\neq k} (a_i N_i)^2.
\] 
The between-class scatter is $\frac{1}{4} (a_k N_k)^2$. Hence, we will have a clustering after projection corresponding to $Y_k$ when 
\[ \sum_{i\neq k} (a_i N_i)^2 < (a_k N_k)^2.\]

We simulated 1,000,000 random vectors $\mathbf{v}$ and checked for which fraction of these vectors there was a $k$ such that $\sum_{i\neq k} (a_i N_i)^2 < (a_k N_k)^2$. We performed this experiment for various values of $r$ and $D$. The results are shown in \cref{fig:projected_separations_plot}. As the number of dimensions increases, the curves appear to approach a limiting curve. As we would predict, the probability of a good separation at $r=1$ approaches zero as $D$ increases. However, for even a modest increase in $r$ (for example to $r=1.2$), the probability of finding a good separation remains at a reasonable level of $10\%$.

\begin{figure}
	\includegraphics[width=0.6\textwidth]{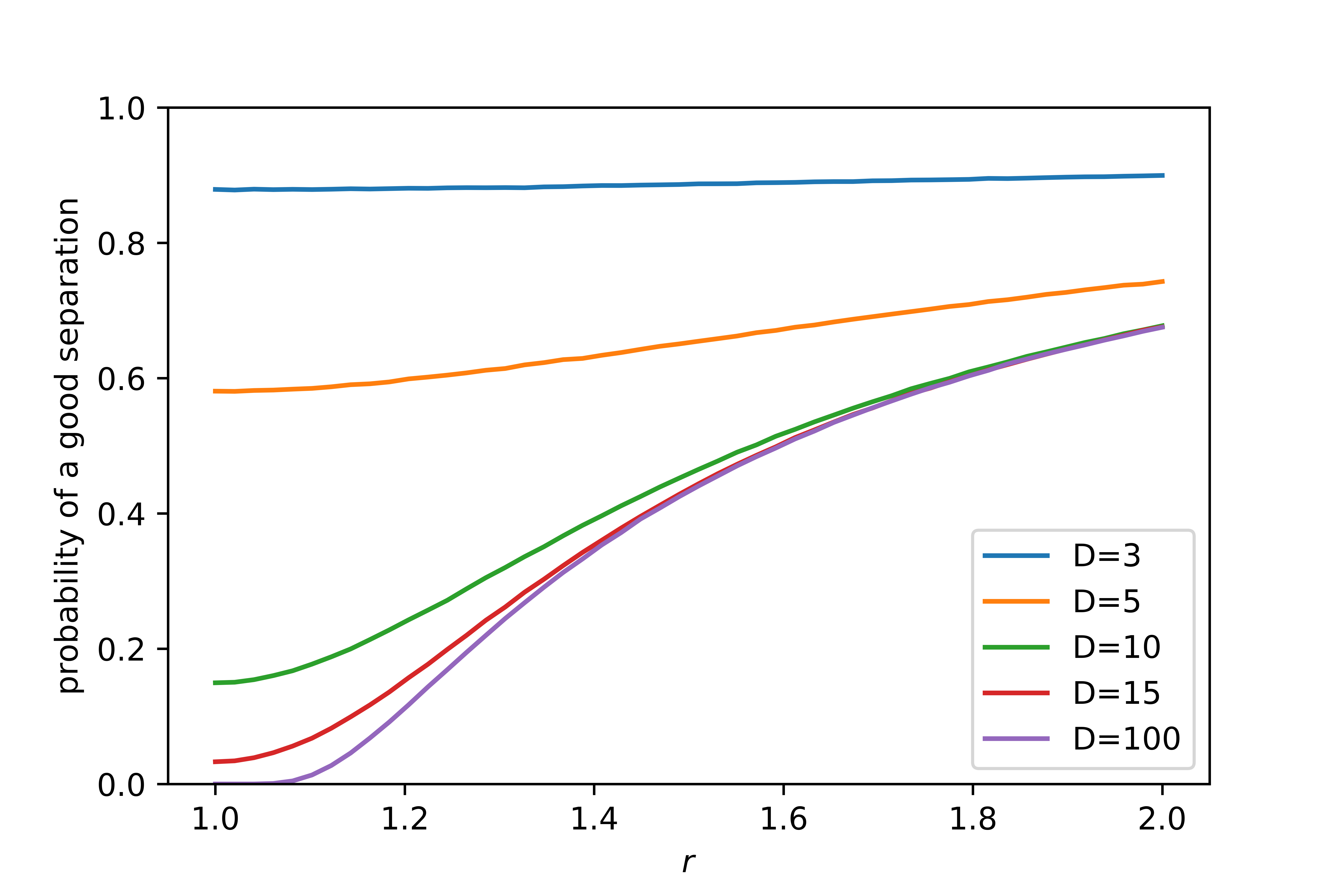}
	\caption{The probability of finding a good separation by random projection, for various values of $r$ and $D$. These curves were empirically computed by simulation, with 1,000,000 trials for each point.}
	\label{fig:projected_separations_plot}
\end{figure}

\subsection{Bounds on the probability of a separation along the largest separation direction}
Above we demonstrated empirically that good separations are fairly likely even when $r$ is not large. In this section, we will provide a bound which explains this behavior in part, based on considering only one of the possible separation directions. Specifically, we will bound the probability that all the points where $Y_d=0$ are separated from the points where $Y_d=1$.

Based on the above discussion, we are considering the probability that 
\[\sum_{i<D} (a_i v_i)^2<(a_dv_d)^2,
\] where $(v_1,\dots, v_D)$ is a random projection vector. We bound this probability by making the simple observation that $a_i$ is an increasing sequence; this is a very loose bound, but it gets a few important points across so we explore it nonetheless. Observe that 
\begin{align*}
P\qty{\sum_{i<D} (a_i v_i)^2<(a_dv_d)^2}&\geq P\qty{\sum_{i<D} (a_{D-1}v_i)^2<(a_dv_d)^2}\\
&= P\qty{\sum_{i<D}v_i^2< r v_d^2}.
\end{align*}

\begin{figure}
    \centering
    \includegraphics[width=0.6\textwidth]{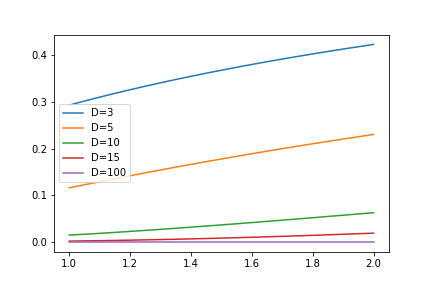}
    \caption{The computed bound on the probability of a separation along the longest axis.}
    \label{fig:chi_square_bound}
\end{figure}

Taking each $v_i$ sampled form a standard normal distribution, this has a familiar form. On the left, $\sum_{i<D}v_i^2$ has a chi-square distribution with $D-1$ degrees of freedom, and the right $rv_d^2$ follows a scaled chi-square distribution with one degree of freedom. Using the formulas for these distributions we can compute this lower bound in terms of an integral; the result is shown in Figure \ref{fig:chi_square_bound}. Notice that this is a poor bound for high $D$; this is due to two factors. First, we are neglecting the other separation directions; each of them contributes some quantity to the probability that one of the directions is well-separated after projection. Second, we are ignoring the effect of each $a_i<a_{D-1}$ which serve to diminish the right-hand side. We bring it up, however, to highlight two facts: (1) that the probability of a separation does increase monotonically toward 1 as $r$ increases and (2) this effect is not explained entirely as scale mixture of two far away clusters as described in Section \ref{sec:two_gaussian}. By contrasting Figures \ref{fig:projected_separations_plot} and \ref{fig:chi_square_bound} we can see the significant contribution of the other projection directions.

\subsection{Note about Data Rescaling}
\label{sec:discussion}
The previous hypercube and the stretched hypercube are very similar, as one is a linear transformation of the other. So any separating projection of the box model could also be found in the hypercube. The only difference is how likely one might find a separating direction of projections when drawing at random. 

A common pre-processing step is to whiten data before examining it. If we were to whiten the stretched hypercube, however, we would get something like the hypercube model. This would make the separations of the many different signals present in the data harder to find. More specifically, whitening  has the effect of amplifying the noise, while diminishing the signal. As a result, there is less room for inaccuracy when choosing a dimension of projection for clustering. Thus this decreases the probability of hitting a separating direction. While one can argue that, in many cases, whitening the data makes sense,
there are many other instances where it does not. Our numerical experiments in Section \ref{sec:experiments} feature such an example.

\section{General Model}
\label{sec:general_model}
We have shown that a mixture of equally weighted Dirac deltas situated on the vertices of a stretched hypercube in ${\mathbb R}^D$ with side length $a_i$, $i=1,\ldots, D$  is such that
\begin{itemize}
\item there exist $D$ different 1D projection directions which result in a bimodal distribution (two deltas) in 1D;
\item the groupings obtained by different projections are different.
\end{itemize}
Mathematically, a random variable from this stretched hypercube model can be written as
\[ X = \left( a_1 B_1 ,\ldots, a_D B_D \right), \quad \text{ for some } a_1,\ldots, a_D \in {\mathbb R}_{>0} \]
where the $B_i$ are Bernoulli random variables with parameter $p_i$. 
If we allow the Dirac deltas to have non-equal weights, then the Bernoulli parameters $p_i$ can take on any values in $[0,1]$ and the Bernoulli random variables $B_i$ do not have to be independent.
If we assume that the side-lengths $a_i=\sqrt{r^{i-1}}$, for some $r$ with $1<r$, then 
\begin{itemize}
\item picking a random projection direction is likely to yield a 1D projected distribution that is bimodal.
\end{itemize}
Furthermore, if $r<2$, then 
\begin{itemize}
\item the binary grouping defined by a bimodal structure in 1D does not correspond to a clustering in the original space.
\end{itemize}
 The stretched hypercube model can be rotated, translated and flipped as  
\[ X = t+ o \left( a_1 B_1 ,\ldots, a_D B_D \right), \quad \text{ for some } t \in {\mathbb R}^D, o\in O(D)  \]
without losing its properties.
An affine transform can also be applied to map the stretched hypercube to a parallelotope in ${\mathbb R}^D$, though this would change the probability of uncovering a binary cluster after a random 1D projection :
\[ X = t+ g \left( a_1 B_1 ,\ldots, a_D B_D \right), \quad \text{ for some } t \in {\mathbb R}^D, g\in GL(D).  \]
Adding noise to the parallelotope leads to a more realistic model. For example, one can replace the Dirac deltas by Gaussians by setting
\begin{equation*}
 X = t+ g \left( a_1 B_1,\ldots, a_D B_D\right) +N  
\end{equation*}
with  $N$ a zero mean normal variable in ${\mathbb R}^D$. 
One can also assume that some dimensions consisting entirely of Gaussian noise with a mean value of $a_i$ by setting the parameter for the corresponding Bernoulli random variable  to either $p_i=1$ or $p_i=0$.
One could argue that these Gaussians form the ``true" groupings of any dataset drawn from such a distribution. But a better description is to say these are subgroups of the different binary groupings found after projection. In other words, they form a fine-scale grouping structure underneath the larger scale (overlapping) grouping structure defined by the 1D projections.

\section{Numerical Experiments}
\label{sec:experiments}
The experiments in this section were performed using scikit-learn \cite{scikit_learn}. 
\subsection{Synthetic image datasets}
The stretched hypercube constructed in Section \ref{sec:hypercube} is a simple, idealized probability model for data without noise. As stated above, the properties of the model are unchanged if it is translated, rotated and flipped. An affine mapping can also be applied to transform it into a parallelotope in ${\mathbb R}^D$, though this does change the probability of obtaining a clustering after projection. Below we show how such a parallelotope distribution can be used to model a sample set of images. This is done by adding a linear combination of dictionary images to a background image. The dictionary images are vectors which span the edges of the parallelotope.  
A change of basis that  maps the dictionary images to the standard axes $(1,0,\ldots ,0)$, $(0,1,0,\ldots ,0 )$ etc., combined with a translation mapping the background image to the origin, would transform the parallelotope back into our initial stretched hypercube model.

More specifically, we construct $24\times 24$ pixel greyscale images by generating random samples in ${\mathbb R}^{24^2}$. 
Let $t\in {\mathbb R}^{24^2}$ be a background image.
Let $k\leq 24^2$ and let $B_1,\ldots, B_k \in {\mathbb R}^{24^2}$ be a set of dictionary images containing different objects. Then an image is modeled as a random vector
\begin{align} \label{vertices}
I= t+\sum_{i=1}^k a_i B_i D_i,
\end{align}
where the $B_i$'s are independent Bernoulli random variables with parameters $p_i$ and the $a_i$'s are the (fixed) scaling factors. The $a_i$ represent the relative darkness of the different dictionary objects on the picture:  the larger the $a_i$, the more dark (prominent) the corresponding shape within the image.
From a perception perspective, 
 the darkness of an object can be related to its importance. For example, very light objects may be considered noise. It that case, it would be unwise to whiten the data, as this would put the noise on equal footing with the signal.

Four images generated by this model are shown in Figure \ref{fig:vertices}. We picked a light gray background image $t$ and the $k=10$ dictionary images shown in Figure \ref{fig:dictionary}. We set all the Bernoulli parameters to $p=\frac{1}{2}$, $a_i=r^{i-1}$ and the stretching parameter to $r=1.5$. Figure \ref{fig:histogram} shows the projection of the images on 10 different lines, namely the lines spanned by the 10 dictionary images $B_i$. Thus clustering number $i$ is based on whether the image does or does not contain the dictionary image $D_i$.  
As one can see, each of the projections form a perfect binary clustering.

Figure \ref{fig:vertices_less_noisy} shows four images generated by adding noise to the previous model. More specifically, the images are generated according to the equation
\begin{align}\label{noisy_vertices}
    I= t+\sum_{i=1}^k a_i B_i D_i + N,
\end{align}
where $\{D_i\}_{i=1}^{k}$ is the dictionary set and $N$ is a random vector with Gaussian distribution: $N \sim \mathcal{N}_{D}\left(0, \sigma^2\mathbbm{1}\right)$. A total of two hundred images were generated using this model. The projection of these images to the same axes as for the previous experiments is shown in Figure \ref{fig:histogram_less_noisy}. As expected, the clusters are now slightly spread out.

\begin{figure}
    \centering
    \includegraphics[width=\textwidth]{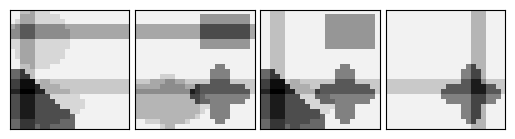}
    \caption{Four of the vertices of a parallelotope according to Formula \ref{vertices}. The background $t$ is a light grey image, the number of elements in the dictionary is $k=10$ and the scaling coefficients (edge lengths) are $a_i=(1.5)^{i-1}$, i=1,\ldots,10.}
    \label{fig:vertices}
\end{figure}

\begin{figure}
    \centering
    \includegraphics[width=\textwidth]{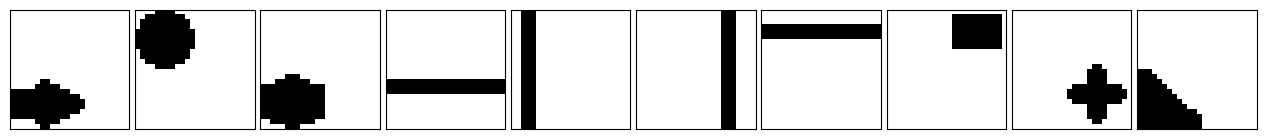}
    \caption{Dictionary of images $\{D_i\}_{i=1}^{10}$ used to generate the vertices of the parallelotope illustrated in Figure \ref{fig:vertices}. }
    \label{fig:dictionary}
\end{figure}

\begin{figure}
    \centering
    \includegraphics[width=\textwidth]{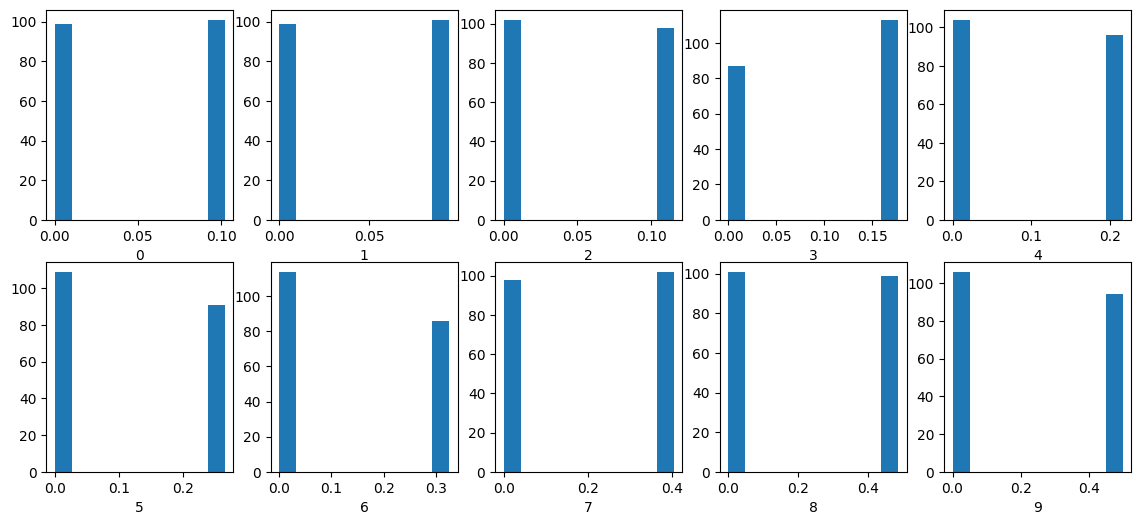}
    \caption{Histogram of image distribution after 10 different projections: $200$ vertices were generated following Formula  \ref{vertices} in a similar manner as the ones in Figure \ref{fig:vertices}. The images were then projected onto the lines spanned by the dictionary images, resulting in 10 perfect (different) clusterings. }
    \label{fig:histogram}
\end{figure}

\begin{figure}
    \centering
    \includegraphics[width=\textwidth]{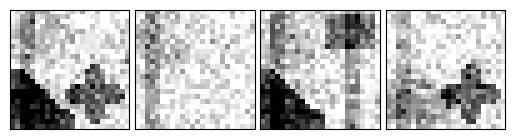}
    \caption{Four noisy vertices of a parallelotope generated according to Formula \ref{vertices}. The background  image $t$ is a uniformly grey image, the number of elements in the dictionary is $k=10$ and scaling coefficients $r_i=(1.5)^{\frac{i-2}{2}}$. The level of noise is set to $\sigma=1$.}
    \label{fig:vertices_less_noisy}
\end{figure}

\begin{figure}
    \centering
    \includegraphics[width=\textwidth]{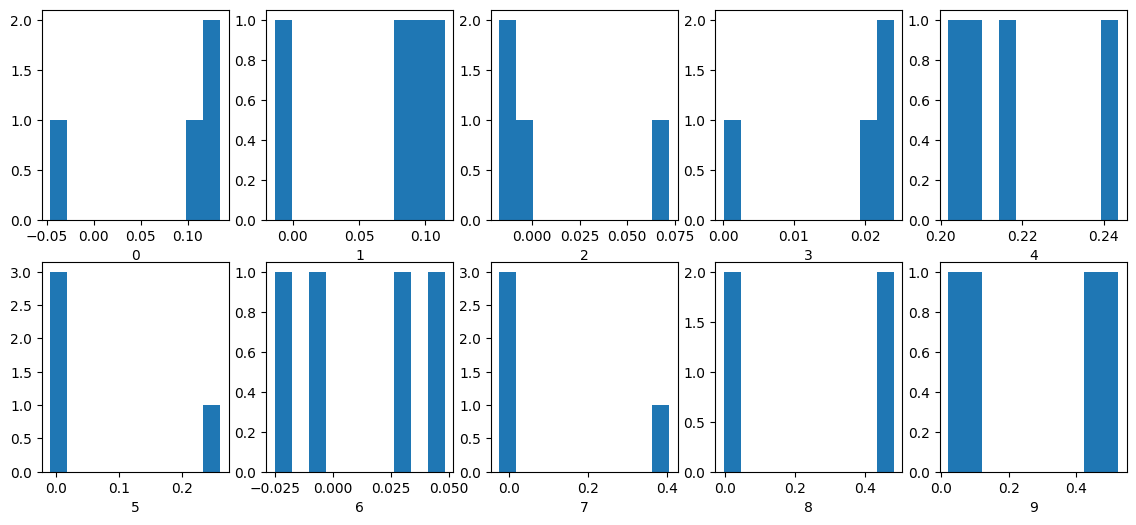}
    \caption{Histogram of image distribution after 10 different projection: $200$ vertices were generated following Formula \ref{noisy_vertices} in a similar manner as the images in Figure \ref{fig:vertices_less_noisy}. The images were then projected onto the lines spanned by the 10 dictionary images, resulting in 10 different clusterings.}
    \label{fig:histogram_less_noisy}
\end{figure}

\color{black}

\subsection{Real image dataset}
Now we consider 1797 digit images available from scikit-learn \cite{scikit_learn} and show that this dataset has a ``noisy" parallelotope structure by showing the existence of projection lines resulting in a binary clustering. 
These are 8x8 greyscale images representing hand-written digits. 
In order to find lines of projection resulting in a binary clustering, we used the  Minimum Description Length criterion in a similar manner as in \cite{Bouman97} to find projection directions resulting in a good binary clustering. The histograms for four of these good projections are shown in Figure \ref{fig:real1}. These four projection directions define the directions of the edges of a parallelotop in ${\mathbb R}^4$; each digit image in the dataset can then be associated to one of the 16 vertices of the parallelotope based on their cluster membership with respect to each of the four projections. The average of the images associated to each vertex is shown in Figure \ref{fig:average_16}. Noe that there are four white average images, which correspond to vertices which are associated to no image in the dataset. Thus the Beroulli parameters controlling the presence of the associated features are not equal to $\frac{1}{2}$, otherwise we would expect each vertex to be associated to roughly the same number of images.

The data in that set appears to have a hierarchical cube structure. For example, the images corresponding to the first vertex in the second row of Figure \ref{fig:average_16}) can be projected to form the bimodal structure shown in Figure \ref{fig:projection_hierarchical}. As one can see from the average images shown in Figure \ref{fig:average_hierarchical}, the two resulting groups of images seem to correspond to images representing shapes that look like a ``4" or a ``6", respectively.
In future work, we plan to develop a systematic way to build a hierarchical hypercube model to encode the structure of real datasets such as this one.

\begin{figure}
  \begin{subfigure}[b]{0.45\textwidth}
    \includegraphics[width=\textwidth]{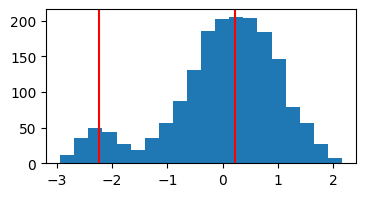}
    \caption{}
  \end{subfigure}
  \begin{subfigure}[b]{0.45\textwidth}
    \includegraphics[width=\textwidth]{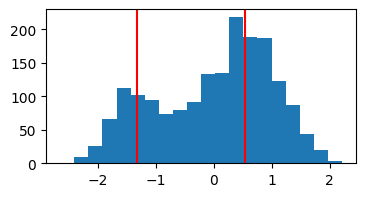}
    \caption{}
  \end{subfigure}
  \begin{subfigure}[b]{0.45\textwidth}
    \includegraphics[width=\textwidth]{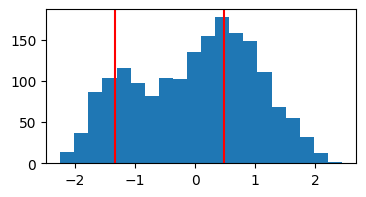}
    \caption{}
  \end{subfigure}
  \begin{subfigure}[b]{0.45\textwidth}
    \includegraphics[width=\textwidth]{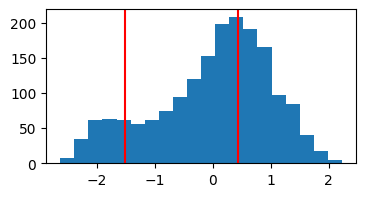}
    \caption{}
  \end{subfigure}
  \caption{Histogram of digit images projected onto 4 different lines. The red vertical line corresponds to mean of the clusters separated based on the Expectation Maximization algorithm.}
  \label{fig:real1}
\end{figure}

\begin{figure}
    \centering
    \includegraphics[width=0.7\textwidth]{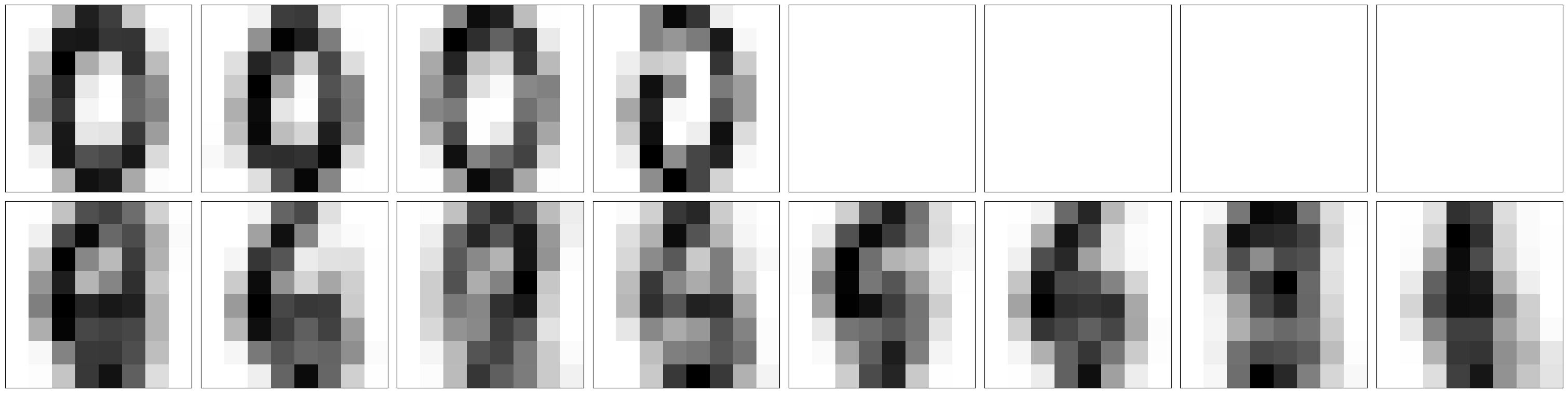}
    \caption{Average of the 16 groups of digit images corresponding to the 16 vertices of the parallelotope defined by the projection directions corresponding to Figure \ref{fig:real1}.}
    \label{fig:average_16}
\end{figure}

\begin{figure}
    \centering    \includegraphics{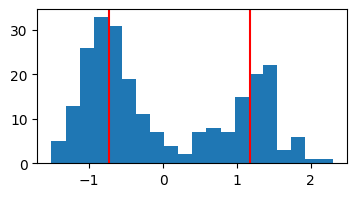}
    \caption{Projection of the group of digit images corresponding to one vertex of the parallelotope defined by the projection directions from Figure \ref{fig:real1}. This binary clustering shows the hierarchical nature of the hypercube model underlying the digit image dataset. }
    \label{fig:projection_hierarchical}
\end{figure}

\begin{figure}
    \centering
    \includegraphics[width=0.7\textwidth]{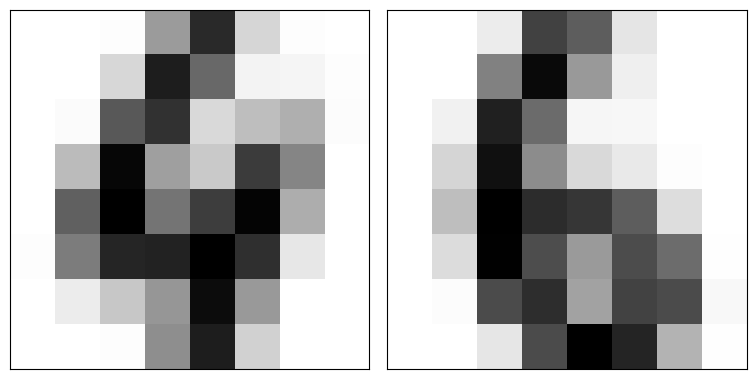}
    \caption{Average of the two groups corresponding to the two clusters from Figure \ref{fig:projection_hierarchical}.}
    \label{fig:average_hierarchical}
\end{figure}

\section{Conclusion and Future Work}
\label{sec:conclusion}

In order to explain previous empirical observations that some real datasets are highly likely clusterable by projection on a random line, we proposed a proof of concept model for the probability density function of high-dimensional data. In order to build the model, we first constructed a skeleton model consisting of a multivariate Bernoulli distribution, in other words placing Dirac deltas on the vertices of a hypercube in ${\mathbb R}^D$.  This first construction illustrates how  seemingly incompatible, yet valid, binary groupings (controlled by different Bernoulli processes) may exist within the same dataset. It also highlights the need to distinguish between the task of grouping high-dimensional data points and that of ``clustering" them, in the sense of finding point accumulations in the original space: points can be grouped meaningfully based on clusters found in a lower-dimensional projection, but the resulting groups are not necessarily clusters (i.e., point accumulations) in the original high-dimensional space.

We then rescaled the multivariate Bernoulli distribution, that is to say we stretched the sides of the hypercube. We showed how certain ways of stretching make the distribution such that projecting it on a randomly chosen  line is likely to yield a bimodal distribution (in 1D). Such easily found clustering directions are associated with the rescaled Bernoulli processes with the largest amplitude (e.g., the darker features in an image). By construction, the prominence of these structures would disappear if the dataset was whitened prior to clustering. 
This construction illustrates a possible reason why certain datasets are easily clusterable after 1D random projection, and the characteristics of the projection directions that are more prominent.

Our construction shows that the multivariate Bernoulli distributions can be rescaled so to be highly likely clusterable after random projection, yet remain so sparse that the Dirac deltas in the original space do not cluster in any way. This further highlights the fact that distributions that are easy to cluster through 1D random projections may not necessarily contain any cluster in the original space, and thus further emphasizes the need to distinguish the task of grouping a set of points versus clustering in the sense of finding point accumulations in space.

After applying a linear transform to the stretched hypercube (transforming it into a parallelotope), one can add noise by replacing the Dirac deltas placed on the vertices by other unimodal distributions (e.g., Gaussians). This adds actual clusters in the original space. 
In the case where the unimodal distributions at the vertices of the parallelotope are Gaussians, our probability model can be viewed as a Gaussian mixture but with a special geometry between the Gaussian positions: a geometry that is controlled by Bernoulli random variables at different scales. If the dataset given contains a relatively small number of points in relation to its dimensionality, the sparsity of the points will make it very difficult to observe these clusters. 
On the other hand, the underlying Bernoulli random processes (i.e., the binary groupings after 1D projection) can be quite easy to identify in 1D, where sparsity is usually no issue.

We are currently investigating extension of this work to the problem of supervised classification \cite{yellamraju2018benchmarks,LarsonBoutin17}, noting the high degree of versatility and good generalization properties of classification by random projection in general \cite{boutin2023optimality}. A model similar to the one proposed here can be associated to a large Rashomon ratio for random projection classification methods, which would guarantee even better generalization properties if the classification problem fits this model \cite{coupkova2024rashomon}.

\section*{Acknowledgements}
We thank Alden Bradford for several significant contributions to this work and Tarun Yellamraju for stimulating discussions.

\bibliographystyle{plain}
\bibliography{bibliography} 

\end{document}